\documentclass{article}

\usepackage[preprint, nonatbib]{nips_2018}

\usepackage[utf8]{inputenc} 
\usepackage[T1]{fontenc}    
\usepackage{hyperref}       
\usepackage{url}            
\usepackage{booktabs}       
\usepackage{amsfonts}       
\usepackage{nicefrac}       
\usepackage{microtype}      

\usepackage{graphicx}
\usepackage{algorithm}
\usepackage{algorithmic}
\usepackage{amsmath}
\usepackage{amsfonts}
\usepackage{amsthm}
\usepackage{amssymb}
\usepackage{enumitem}
\usepackage{kotex}
\usepackage{subfig}
\usepackage{xcolor}

\definecolor{orange}{RGB}{255,127,0}

\newcommand{\eqnref}[1]{Eqn.~(\ref{#1})}
\newcommand{\algref}[1]{Alg.~\ref{#1}}

\newcommand{\propref}[1]{Prop.~\ref{#1}}

\newcommand{\E}{\mathbb{E}}

\newcommand{\R}{\mathbb{R}}
\newtheorem{thm}{Theorem}

\newtheorem{prop}{Proposition}

\newcommand{\btheta}{\boldsymbol{\theta}}
\newcommand{\bx}{\boldsymbol{x}}

\newcommand{\mcD}{\mathcal{D}}

\title{Stochastic Doubly Robust Gradient}

\author{
\vspace*{.05in}
Kanghoon Lee$^{1}$\thanks{Equal contribution.} ~~ 
Jihye Choi$^{2}$\footnotemark[1]\hspace{.055in}\thanks{Work done while at T-Brain.} ~~
Moonsu Cha$^{3}$\footnotemark[2] ~~ 
Jung-Kwon Lee$^{3}$\footnotemark[2] ~~ 
Taeyoon Kim$^{1}$ \\ 
\vspace*{.02in}
$^{1}$ T-Brain, AI Center, SK telecom ~
$^{2}$ Carnegie Mellon University ~
$^{3}$ Superb AI \\
\texttt{khlee@sktbrain.com}, ~
\texttt{jihyec@andrew.cmu.edu}, ~
\texttt{\{mscha,jklee\}@spb.ai}, \\
\texttt{oceanos@sktbrain.com}
}

\begin{document}

\maketitle

\begin{abstract}
When training a machine learning model with observational data, it is often encountered that some values are systemically missing. 
Learning from the incomplete data in which the missingness depends on some covariates may lead to biased estimation of parameters and even harm the fairness of decision outcome.
This paper proposes how to adjust the causal effect of covariates on the missingness when training models using stochastic gradient descent (SGD).
Inspired by the design of doubly robust estimator and its theoretical property of double robustness, we introduce stochastic doubly robust gradient (SDRG) consisting of two models: weight-corrected gradients for inverse propensity score weighting and per-covariate control variates for regression adjustment.
Also, we identify the connection between double robustness and variance reduction in SGD by demonstrating the SDRG algorithm with a unifying framework for variance reduced SGD.
The performance of our approach is empirically tested by showing the convergence in training image classifiers with several examples of missing data.


\end{abstract}
\section{Introduction}

The missing data problem is commonly encountered when training a machine learning model with real world data: unlike in the case of clean-cut experimental data, one or more covariates are often missing in recorded observations. 
Learning from the incomplete data may introduce an undesirable bias, especially when the missingness mechanism is not completely at random. 
More specifically, if the missingness depends on some covariates (e.g., gender, age, religion, and race) involved in generating the data, the estimation based on these unequally collected observations can be significantly different from the ideal result. This does not only interfere with the consistency in the learning process, but may also have a profound effect on the fairness of learning outcome. 

To mitigate this problem, one may want to infer the causal effect of covariates on the missingness mechanism when training models. 
\emph{Doubly robust estimator} \cite{robins1994estimation,rotnitzky1998semiparametric}, first introduced in the area of observational study, has been known as an effective method to deal with such causal missingness and still remains popular \cite{BangRobins.Biometrics05,KangEtal.SS07,RotnitzkyEtal.Biometrika12,han2013estimation,zubizarreta2015stable}. 
It employs two well-known approaches, regression adjustment and inverse propensity score weighting, and by its interesting theoretical property, guarantees that the estimate remains unbiased as long as either of the two is specified correctly. 
In recent years, the double robustness has emerged in wide range of machine learning areas including covariate shift \cite{ReddiEtal.AAAI15}, adversarial training \cite{Kallus.arXiv18}, and reinforcement learning \cite{DudikEtal.ICML11,JiangLi.ICML16,ThomasBrunskill.ICML16,FarajtabarEtal.ICML18}. 

In this paper, we introduce the concept of doubly robust estimator to stochastic gradient descent (SGD) to correct the bias induced by the causal missingness in training data.
Our approach, namely \textit{stochastic doubly robust gradient (SDRG)}, consists of per-covariate control variates and weight-corrected gradients that serve as the methods for regression adjustment and inverse propensity score weighting, respectively. 
To the best of our knowledge, SDRG is the first SGD algorithm that provides the property, double robustness.

Recently, the use of control variate methods have been excessively studied in the literature of variance reduction of SGD for accelerating the convergence \cite{roux2012stochastic,defazio2014saga,JohnsonZhang.NIPS13,wang2013variance,reddi2016stochastic}.
Among these works, stochastic variance reduced gradient (SVRG) \cite{JohnsonZhang.NIPS13} and SAGA \cite{defazio2015new}, which is an unbiased estimate of stochastic average gradient (SAG) \cite{roux2012stochastic}, are closely related to SDRG in that both of them and SDRG can be viewed within a generic framework for control variates of variance reduced SGDs \cite{reddi2015variance}, as SDRG involves the use of per-covariate control variates.
Finding connections between the ways control variates are used, we will show how pursuing double robustness in gradient estimation is aligned with reducing the variance.

Although SDRG and the variance reduced SGDs look similar to each other, there is a notable difference in the situation they can handle: each gradient estimate should be weighted unequally to reflect the missingness in training data and the weight-corrected gradients in SDRG are devised to address this need, whereas the aforementioned variance reduction methods can only consider the equal contribution of individual gradients.
The SDRG algorithm can be straightforwardly applied to practical scenarios such as class imbalance problem, as we demonstrate in this paper considering some contextual information (such as class-label or any kind of tag) of training data as the covariates.




In summary, our contributions are as follows.
\renewcommand\labelitemi{\tiny$\bullet$}
\begin{itemize}[leftmargin=*] \itemsep1pt
  \item We propose the first doubly robust SGD algorithm, called SDRG, and demonstrate that SDRG can be devised in much the same way of SAGA and SVRG.
  \item We define per-covariate momentum functions as control variates of SDRG, and show that it does not require to either periodically calculate (as SVRG) or store (as SAGA) the full gradients.
  \item We provide a relation between SDRG and momentum, which is a much more direct derivation than the previous relationship presented in \cite{roux2012stochastic}.
  \item We experimentally show the performance of SDRG in training image classifiers with class-imbalanced MNIST and fashion-MNIST datasets since they are simple, yet commonly arising form of missing data problems.
\end{itemize}

To clarify, we remark that our work is not aligned with the approaches that employ non-uniform importance sampling for variance reduction in SGD based \cite{NeedellEtal.NIPS14,zhao2015stochastic,KatharopoulosFleuret.ICML18,kern2016svrg++,shen2016adaptive}. Rather than proposing a sampling criteria, the purpose of our work is to develop a robust learning algorithm when the weights are already determined with regards to the causal missingness.

\section{Background and Related Work}
A principled optimization problem in modern machine learning is that of the \textit{finite-sum} form: minimization of an objective function $f(\btheta)$ that are naturally expressed as a summation over a finite set of data $\mcD = \{ x_i \}_{i=1}^n$, which is described as,
\begin{align}
  \label{eq:finite-sum}
  \min_{\btheta \in \mathbb{R}^d} f(\btheta) 
  := \frac{1}{n} \sum_{i=1}^n w_i f_i(\btheta)
\end{align}
where $\btheta$ is a parameter to be optimized over, each term $f_i(\btheta)$ contributes with the weight $w_i$, and $w_i = 1$ for typical setup. 
Such objective  in \eqnref{eq:finite-sum} commonly appears in the empirical risk minimization framework where the objective is the average of losses computed over the data in $\mcD$, that is, $f_i(\btheta) := L(\btheta; x_i)$

Stochastic gradient descent (SGD) is a method of choice to deal with such optimization problem. It iteratively updates the design parameter as follows: for each training iteration $t=1,2,...,T$,
\begin{align}
  \nonumber
  \btheta^{t+1} &= \btheta^{t} - \eta ~ \Delta \btheta^{t} \\
  \nonumber
  \Delta \btheta^{t} &= \nabla f_{i_t}(\btheta^{t})
\end{align}
where $\eta>0$ is a learning rate, $i_t \in \{1,...n\}$ and $f_{i_t}(\cdot)$ is the loss computed with $x_{i_t}$ which is drawn iteratively from a training set $\mathcal{D}$.

In recent years, a class of algorithms to improve the convergence of SGD by reducing the variance of the estimates has been proposed \cite{roux2012stochastic,JohnsonZhang.NIPS13,wang2013variance,defazio2014saga}.
Especially, Reddi et al. \cite{reddi2015variance} provides a formal unifying framework as \algref{alg:generic_cv} for stochastic variance reduction methods proposed in the literature, including SVRG \cite{JohnsonZhang.NIPS13}, SAGA \cite{defazio2014saga}, and SAG \cite{roux2012stochastic}.
The basic idea behind the variance reduction methods is to augment the gradient with a control variate and its expectation as,
\begin{align}
  \label{eq:sgd-vr}
  \Delta \btheta^{t}
  & = \nabla f_{i_t}(\btheta^{t}) - g_{i_t}(\tilde\btheta) 
  + \E[g_{i}(\tilde\btheta)]
\end{align}
where $\tilde{\btheta}$ is an approximation of $\btheta$.
The resulted estimate is unbiased, and has smaller variance if $g_{i_t}(\tilde{\btheta})$ has a high correlation with the target estimate $\nabla f_{i_t}(\btheta^t)$. 

\begin{algorithm}[t]
  \caption{
  Generic Control Variate Method in SGD}
  \label{alg:generic_cv}
  \textbf{Initialize: } 
  $\btheta^{0} \in \mathbb{R}^d$, $\tilde\btheta = \btheta^{0}$,
  $\forall i: g_i(\tilde\btheta) = 0$,
  $\eta>0$

\begin{algorithmic}[1]
  \FOR{$t = 0, ..., (T-1)$}
    \STATE (Uniform-) randomly pick an $i_{t} \in \{1,...,n\}$
    \STATE Compute the surrogate estimation of $\Delta \btheta^{t}$: \\
    \quad $\Delta \btheta^{t}
    = \nabla f_{i_t}(\btheta^{t}) - g_{i_t}(\tilde\btheta)
    + \frac{1}{n} \sum_{i=1}^n g_{i_t}(\tilde\btheta)$
	\STATE Update the parameter $\btheta^{t+1}$: \\
    \quad $\btheta^{t+1} \leftarrow \btheta^{t} - 
    \eta ~ \Delta \btheta^{t}$
    \STATE \label{alg:sch} Update the schedule 
    $g_{i_t}(\cdot)$ and/or $\tilde\btheta$: \\
	\quad Option I (SVRG): 
    update  $g_i(\cdot)$, $\tilde\btheta$ using \eqnref{eq:svrg-cv} \\
	\quad Option II (SAGA): 
    update $g_{i_t}(\tilde\btheta)$ using \eqnref{eq:saga-cv}
  \ENDFOR
  \RETURN $\btheta^T$
\end{algorithmic}
\end{algorithm}

As studied in \cite{reddi2015variance}, the mechanisms of updating control variates, $\{g_i(\tilde{\btheta})\}_{i=1}^n$ can be arranged within the unifying framework  (see line 5 of \algref{alg:generic_cv}) for the well-known variance reduction methods:

\paragraph{SVRG}
Parameter $\tilde{\btheta}$ is updated after every $m$ iterations as
\begin{align}
  \nonumber
  g_i(\cdot) 
  & = \nabla f_i(\cdot) ~ \textnormal{ for all } i \\
  \label{eq:svrg-cv}
  \tilde\btheta & =
  \begin{cases}
    ~ \btheta^{t}
    & \textnormal{ if $t \textnormal{ mod } m = 0$} \\
    ~ \tilde\btheta
    & \textnormal{ otherwise}
  \end{cases}
\end{align}

\paragraph{SAGA} The gradients of all functions $\{\nabla f_i(\btheta)\}_{i=1}^n$ are kept in memory, and one of them corresponding to the training instance is updated at every iteration as
\begin{align}
  \label{eq:saga-cv}
  g_i(\tilde\btheta) & =
  \begin{cases}
    ~ \nabla f_i(\btheta^{t})
    & \textnormal{ if $i = i_t$} \\
    ~~~ g_i(\tilde\btheta)
    & \textnormal{ otherwise} ~.
  \end{cases}
\end{align} 
For SAG, the only difference with SAGA is that the line 3 of \algref{alg:generic_cv} is changed into,
\begin{align}
  \label{eq:sag}
  \Delta \btheta^{t}
    = \frac{1}{n} \left[ \nabla f_{i_t}(\btheta^{t}) - g_{i_t}(\tilde\btheta) \right]
    + \frac{1}{n} \sum_{i=1}^n g_{i_t}(\tilde\btheta)
\end{align}
One may notice that SAG update rule does not exactly fit in the formulation of \eqnref{eq:sgd-vr} since the last term in \eqnref{eq:sag} does not become an expectation of the control variate by the scale of $1/n$. However, we categorize SAG as control variate-based variance reduction methods along with other methods, since they are all similar in the sense of incorporating an additional parameter to reduce the variance of estimates.


The aforementioned approaches are originally under strong convexity assumptions and has been extended to non-convex optimization problems \cite{allen2016variance,reddi2016stochastic}. 
Asynchronous \cite{reddi2015variance,meng2016asynchronous,huo2017asynchronous}, proximal \cite{xiao2014proximal,allen2017katyusha} and accelerated variants have also been proposed. 


\section{Stochastic Doubly Robust Gradient}
\label{control-variate}

Before to introduce our main algorithm, we begin by formalizing the notion of weighted finite-sum problem that we are interested in and introduce notation that we use throughout this paper.

\subsection{Problem Setting}

We consider the cases where the individual loss term,  $f_i$, in \eqnref{eq:finite-sum} contributes unequally to the optimization, and thus, should be weighted differently according to certain criteria (i.e., $w_i \neq 1$). 
The weight $w_i$ is defined by the generic importance sampling literature \cite{WeissEtal.AAAILBD13,ThomasBrunskill.AAAI17,DoroudiEtal.UAI17}: 
when training and testing data come from different distributions, it can be specified to correct the difference.
To elaborate, we are only given a collected set of data from a distribution $q$ (which we call the \textit{sampling distribution}).
We originally intend to compute the expected loss over some \textit{proper} distribution $p$ (which we refer to as the \textit{target distribution}), that is, $\E_{p}[f(\btheta;\bx)]$, $\bx \sim q$.
Since we may not have direct access to $p$, however, we want to do a finite-sum approximate to the expectation over samples $\{\bx_i\}_{i=1}^n, ~ \bx_i \in \mathcal{\R}^d$ drawn from the distribution at hand: 
\begin{align}
  \label{eq:finite-sum-is}
  \nonumber
  \E_{p}[f(\btheta)] 
  & = \E_{q} \left[
  \frac{p(\bx)}{q(\bx)} f(\btheta;\bx)
  \right] \\
  & = \frac{1}{n} \sum_{i=1}^n 
  \frac{p(\bx_{i})}{q(\bx_{i})} f_i(\btheta)
  ~ , ~~~ \bx_{i} \sim q
\end{align}
where $p(\bx_{i})/q(\bx_{i})$, the ratio between target distribution and sampling distribution is regarded as the weighting factor $w_i$.

\subsection{Weight-Corrected Gradient with Variance Reduction}

To solve the weighted finite-sum problem as \eqnref{eq:finite-sum-is}, in the standard SGD algorithms, the $t$-th iteration involves picking an instance from sampling distribution $q$ over all instances and updates parameters as
\begin{align}
  \label{eq:sgd-iw}
  \Delta \btheta^{t}
  & = w_{i_t} \nabla f_{i_t}(\btheta^{t})
\end{align}
where $w_{i_t}:=p(\bx_{i_t})/q(\bx_{i_t})$ is per-sample importance weight.
This weight-corrected gradient can be thought of as an inverse propensity score estimate, and we call it \textit{importance weighted SGD}.

Within this literature of doubly robust estimator, our goal is to reduce the variance of stochastic gradient algorithm by introducing a control variate method to accelerate SGD.
Given \eqnref{eq:sgd-vr} and \eqnref{eq:sgd-iw}, an intuitive approach to employ a control variate to stochastic weighted gradient descent is as follows,
\begin{align}
  \label{eq:weighted-cv}
  \Delta \btheta^{t}
  & = w_{i_t} \nabla f_{i_t}(\btheta^{t})
  - w_{i_t} g_{i_t}(\tilde{\btheta})
  + g(\tilde{\btheta})
\end{align}
where $\tilde\btheta$ is required to be highly correlated with $\btheta$ and $g(\tilde{\btheta}) := \E_{p}[g_i(\tilde{\btheta})]$. 
The resulted estimate is unbiased as the stochastic weighted gradient in \eqnref{eq:sgd-iw} with reduced  variance. 



The constructed gradient estimator in \eqnref{eq:weighted-cv} involves two variables: per-sample importance weight $w_{i_t}$ and control variate $g_{i_t}$.
In other words, the estimation accuracy of our approach relies on how correctly the two could be specified. 
From this perspective, we see an advantage of our formulation by observing that either one of two models needs to be correctly specified to obtain an unbiased estimator:


Seen in a broader context, such distinguishing property of our formulation in \eqnref{eq:weighted-cv} arises in many areas. 
Indeed, in the area of observational study, the property called \textit{double robustness} has been well studied \cite{robins1994estimation,BangRobins.Biometrics05,KangEtal.SS07,RotnitzkyEtal.Biometrika12}: doubly robust (DR) estimators involves models for both the propensity score and the conditional mean of the outcome, and remain consistent even if one of those models (but not both) is misspecified.
By observing that the constitution and the property of DR estimator are similar with that of our constructed gradient estimator, we see an opportunity to bring the insights of DR estimation into stochastic gradient optimization. 

\begin{thm}
  \label{thm:dr-estimator}
  The weight-corrected gradient with variance reduction in \eqnref{eq:weighted-cv} satisfies the double robustness. Thus, \eqnref{eq:weighted-cv} is a doubly robust gradient estimator.
\end{thm}
\begin{proof}
  First, we rewrite \eqnref{eq:weighted-cv} as follows:
  \begin{align*}
    \Delta \btheta^{t} 
    & = w_{i_t} \nabla f_{i_t}(\btheta^{t}) 
    - \left\{ 
    w_{i_t} g_{i_t}(\tilde{\btheta}) - g(\tilde{\btheta}) 
    \right\} \\
    & = w_{i_t} \left\{
    \nabla f_{i_t}(\btheta^{t}) - g_{i_t}(\tilde{\btheta})
    \right\} + g(\tilde{\btheta})
  \end{align*}
  Since $\E[ w_{i_t} g_{i_t}(\tilde{\btheta})]$ is equal to $g(\tilde{\btheta})$ and $\E[\nabla f_{i_t}(\btheta^{t}) - g_{i_t}(\tilde{\btheta})] = 0$, we can observe that $\Delta \btheta^{t}$ is unbiased if either $w_{i_t}$ is accurately estimated or the model $g_{i_t}(\tilde\btheta)$ is correct.
  Therefore, we can conclude that \eqnref{eq:weighted-cv} is a doubly robust estimator. \\
\end{proof}


In the statistics community, and particularly in causal inference settings, DR estimators provide an estimation on average causal effect from observational data, adjusting appropriately for confounders. 
The ability of DR estimators to taking account to the causal effect of confounders can be also useful in the general machine learning literature. 
For instance, in supervised machine learning, the goal is to seek a function $h: X \rightarrow Y$, given $n$ pairs of inputs and corresponding target outputs $\{\bx_i,y_i\}_{i=1}^n$.
Meanwhile, there often exist contextual information associated with instances $\bx_i$, and it is not directly used to compute the objective (loss) value but may indirectly influence the process of learning the relation between $\bx_i$ and $y_i$.
In that case, one may want to address causal effect of contextual information in the process of learning the relation between $\bx_i$ and $y_i$.
By regarding the contextual information as confounding factors in the approximation of the expected loss over training set, 
we propose a method to adjust the causal effect of contextual information that may confound the gradient estimation.

\subsection{Confounded Mini-Batch Gradient}
We maintain different models that estimate two key parameters of \eqnref{eq:weighted-cv} – importance weights and control variates, conditioned on each configuration of contextual information.
In practice, the contextual information could represent a class-label or any tag associated with $x_i$ \cite{Gopal.ICML16}. 
In this paper, we decide to use class-label as contextual information as they are almost always available. 
Given a finite set of class-labels as the observed contextual information, we confine our interest to class imbalance problem, one of the most common scenario we may encounter in classification tasks.  
By taking a mini-batch variant of our estimator, we view the class imbalance problem in the perspective of importance weighting: if the data set collected in the mini-batch is sampled from highly skewed distribution, we want to correct the difference between the skewed distribution and the target distribution which is assumed to be uniformly balanced in terms of classes. 
Let $I_t$ be the set of indices of mini-batch instances at training iteration $t$, $I_{t, c}$ is a disjoint subset of $I$, whose instances belong to class $c$, $w_c(\cdot)$ is per-class model for estimating importance weights and $g_{t,c}(\cdot)$ is per-class control variates. Then, our proposed algorithm for class imbalance problem is described as, 
\begin{align}
  \nonumber
  \Delta \btheta^{t}
  & = \frac{1}{C} \sum_{c=1}^C \Bigg[
  \frac{1}{|I_{t,c}|} \sum_{i_{t,c} \in I_{t,c}}
  w_c(i_{t_c}) ~ \nabla f_{i_{t,c}}(\btheta^{t})
  - \frac{1}{|I_{t,c}|} \sum_{i_{t,c} \in I_{t,c}}
  w_c(i_{t_c}) ~ g_{i_{t,c}}(\tilde\btheta_c) + g(\tilde\btheta_c)
  \Bigg]
\end{align}
We call this method \textit{stochastic doubly robust gradient (SDRG)}.
In practical implementations, an intuitive way of setting importance weights in the above setting is to compute the proportion of the number of instances that belong to each class-label over the mini-batch size.

We show that the existing variance reduction methods are related to SDRG in the sense their mechanism of updating control variates and computing the expectation of them (line 5 of \algref{alg:generic_cv} can be directly used in SDRG update rule) as follows:
\begin{align*}
  \tilde\btheta_{c} 
  & = \begin{cases}
    ~ \btheta^{t}
    & \textnormal{ if $t \textnormal{ mod } m = 0$} \\
    ~ \tilde\btheta_{c}
    & \textnormal{ otherwise}
  \end{cases}
\end{align*}
where $m$ is the parameter update frequency and $\tilde\btheta_c$ is initialized by $\btheta^{0}$ for all $c$, and
\begin{align*}
  g^t(\tilde\btheta_{c})
  & = \eta ~ \frac{1}{|I_{t,c}|} \sum_{i_t \in I_t} 
  \nabla f_{i_{t,c}}(\btheta^{t}) 
  + \gamma ~ g^{t-1}(\tilde\btheta_{c})
\end{align*}
where $g^0(\tilde\btheta_c)$ is initialized by 0 for all $c$.

\section{Relation to Momentum}
\label{doubly-robust}

In practical implementations, it is natural to take $\tilde\btheta$ as the average or a snapshot from the past iterations \cite{JohnsonZhang.NIPS13}. However, we propose to set $\tilde\btheta$ as a geometric weighting of previous gradients and by doing so, show an simple analysis on the relation between the above formulation and momentum optimizer \cite{Qian.NN99}: the proposed control variate method which is described as,
\begin{align*}
  \btheta^{t+1} 
  & = \btheta^{t} - \Delta \btheta^{t} \\
  & = \btheta^{t} - \left(
  w_{i_t} \nabla f_{i_t}(\btheta^{t})
  - w_{i_t} g_{i_t}(\tilde{\btheta}) 
  + g(\tilde{\btheta})
  \right)
\end{align*}
reduces to the formulation of momentum under a special setting where $g_{i_t}(\tilde\btheta) = g(\tilde\btheta) = \frac{\gamma}{1-\eta} \Delta \btheta^{t-1}$ and $w_{i_t} = \eta ~ (\neq 1)$, as follows,
\begin{align}
  \label{eq:momentum}
  \btheta^{t+1} & = \btheta^{t} - \left(
  \eta ~ \nabla f_{i_t}(\btheta^{t}) + \gamma ~ \Delta \btheta^{t-1} 
  \right)
\end{align}
where $\gamma$ is a momentum coefficient. 
In other words, the control variate method with constant importance weight has the exact same formulation with momentum update rule.

\begin{prop}
  \label{prop:momentum}
  The weight-corrected gradient with variance reduction in \eqnref{eq:weighted-cv} generalizes the momentum update rule in \eqnref{eq:momentum} into the cases of importance weighting.
\end{prop}

From the perspective of the classic control variate schemes that involves an additional parameter and its expectation, momentum method can be regarded as a biased estimator since the expectation of weighted control variates $w_{i_t} g_{i_t}$ does not correspond to $g(\tilde\btheta)$.
The interpretation on the momentum as a biased estimator can be also find in SAG of \eqnref{eq:sag} and from this observation, we see a connection between momentum and SAG using our control variate formulation. 
It is noteworthy that there has been an attempt to find a relation between SAG and momentum optimizer:


in the work of \cite{roux2012stochastic}, \eqnref{eq:sag} and \eqnref{eq:momentum} methods can be expressed in the following formulation,
\begin{align*}
  \textnormal{SAG}
  & :\quad
  \btheta^{t+1}
  = \btheta^{t} + \eta \sum_{j=1}^t 
  S(j,i_{1:t}) ~ \nabla f_{i_t}(\btheta^{j}) \\
  \textnormal{Momentum}
  & :\quad 
  \btheta^{t+1}
  = \btheta^{t} + \eta \sum_{j=1}^t 
  (\gamma \eta)^{t-j} ~ \nabla f_{i_t}(\btheta^{j})
\end{align*}
where $S(i,i_{1:t})$ is the selection function and equal to $1/n$ if $j$ corresponds to the last iteration where $j = i_t$ and is set to $0$ otherwise. Roux et al. \cite{roux2012stochastic} finds a connection of SAG on momentum method by showing that they can be written in a similar formation. 
However, we provide an simpler but stronger analysis by showing a direct connection of SAG on momentum method by proving that they reduce to the exactly same form of equation (\propref{prop:momentum}).

\section{Experiments}
In this section, we experiment with SDRG in comparison with importance weighted gradient descent for classification task with MNIST \cite{Lecun.98} and Fashion-MNIST \cite{xiao2017fashion} datasets.
Both of MNIST and Fashion-MNIST are well balanced for 10 classes, but as we want to test the convergence rate in the setting of class imbalance problem, we modify the sampling distribution to make the sampled instances in mini-batches to be highly unbalanced in terms of class-label.
Assuming that we want the training instances are uniformly distributed for all classes, we correct the gradient computed over the mini-batch samples drawn from the skewed distribution by employing importance weighting.
In such cases where importance weighting is employed, we demonstrate that our SDRG which augments the importance weighted stochastic gradient with control variates whose expectation is replaced by a momentum, shows empirical improvements on the convergence rate.

\begin{figure*}
  \centering
  \begin{tabular}{cc}
    \subfloat[\label{fig:mnist}][Sampling distribution is skewed consistently for a single class (i.e. class 0)]
    {\hspace{-5mm}
    \includegraphics[width=.49\linewidth]{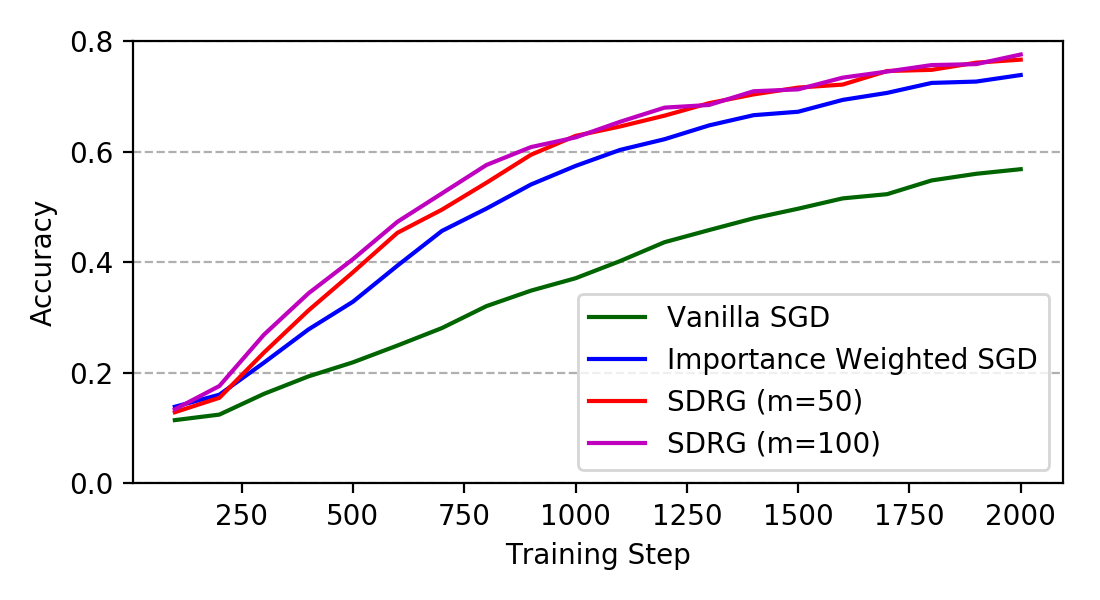}} 
    ~&~
    \subfloat[\label{fig:asd}][Sampling distribution is skewed for the classes in turn (i.e. class 0, 1, 2, ...)]
    {\includegraphics[width=.49\linewidth]{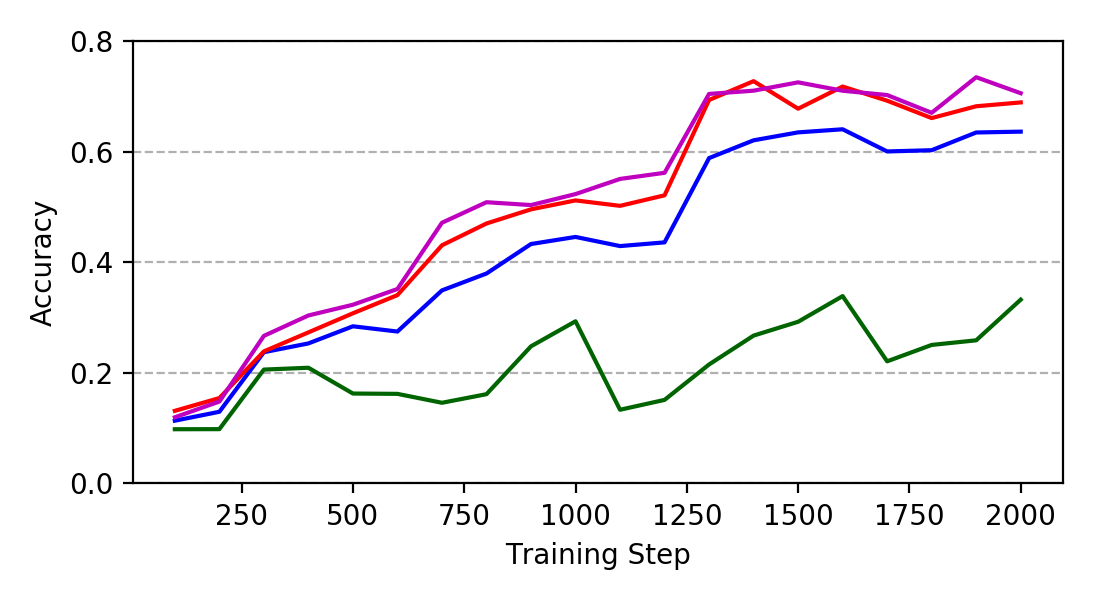}}
  \end{tabular}
  \caption{Accuracy vs training timestep in MNIST}
\label{fig:mnist_results}
\end{figure*}

\paragraph{MNIST} MNIST is a large database of handwritten digits from 0 to 9 that is commonly used to evaluate various machine learning algorithms. The images are grayscale with size of 28 x 28 pixel and they are uniformly balanced for all 10 classes.

We devise two different mechanisms to generate settings for the class imbalance problem: (a) first, the sampling distribution is skewed consistently for a single class during the entire training process. For instance, the probability for the instances from class 0 to be sampled is forced as 0.8, where the instances from the rest of 9 classes are sampled uniformly. (b) Otherwise, the class to be skewed with the sampling probability of 0.8 is selected in turn from class 0, 1, 2, ... to 9. 

We test the performance of SDRG update rule which can be specifically written as the follows for 10 class classification task:
Let $I_t$ be a set of indices for mini-batch samples at $t$ th training iteration, and $I_{t,c}$ be the disjoint subset of $I_t$ that is the set of indices of instances that belong to class-label $c$.
\begin{align*}
  \Delta \btheta^{t}
  = \frac{1}{10} \sum_{c=0}^9 \Bigg[
  \frac{1}{|I_{t,c}|} \sum_{i_t \in I_t}
  \nabla f_{i_{t,c}}(\btheta^{t}) 
  - \frac{1}{|I_{t,c}|} \sum_{i_t \in I_t}
  g_{i_{t,c}}(\tilde\btheta_{c}) + g^t(\tilde\btheta_{c})
  \Bigg]
\end{align*}

We compare the performance of SDRG with two algorithms: importance weighted gradient descent which is described as,
\begin{align*}
  \Delta \btheta^{t}
  & = \frac{1}{10} \sum_{c=0}^9
  \frac{1}{|I_{t,c}|} \sum_{i_{t,c} \in I_{t,c}}
  \nabla f_{i_{t,c}}(\btheta^{t})
\end{align*}
and vanilla SGD.

To compare SDRG and other algorithms, we train a neural network (with one fully-connected hidden layer of 100 nodes and ten softmax output nodes) using cross-entropy loss with mini-batches of size 20 and learning rate of 0.01. 
We evaluate the performance of SDRG with different frequency of updating $\tilde\btheta$: $m =50, 100$. For momentum parameters, $\gamma =0.9$ and $\eta =0.1$ are used.
And we add two extra parameters $\alpha$ ($0.5$ for (a) and $1.5$ for (b)) and $\beta$ ($1.5$ for (a) and $0.5$ for (b)) for additional weighting sample gradients and control variate function, respectively. These parameters are corresponding to weights of $w_i=\eta$ and $g_{i_t}(\tilde\btheta)=g(\tilde\btheta)=\frac{\gamma}{1-\eta} \Delta \btheta^{t-1}$ in \eqnref{eq:momentum}. The results are all generated by taking the average from 20 runs of experiments. The confidence intervals are too insignificant to be noted and we decided not to include them in the figures.

\begin{figure*}
  \centering
  \begin{tabular}{cc}
    \subfloat[\label{fig:fmnist}][Sampling distribution is skewed consistently for a single class (i.e. class 0)]
    {
    \hspace{-5mm}
    \includegraphics[width=.49\linewidth]{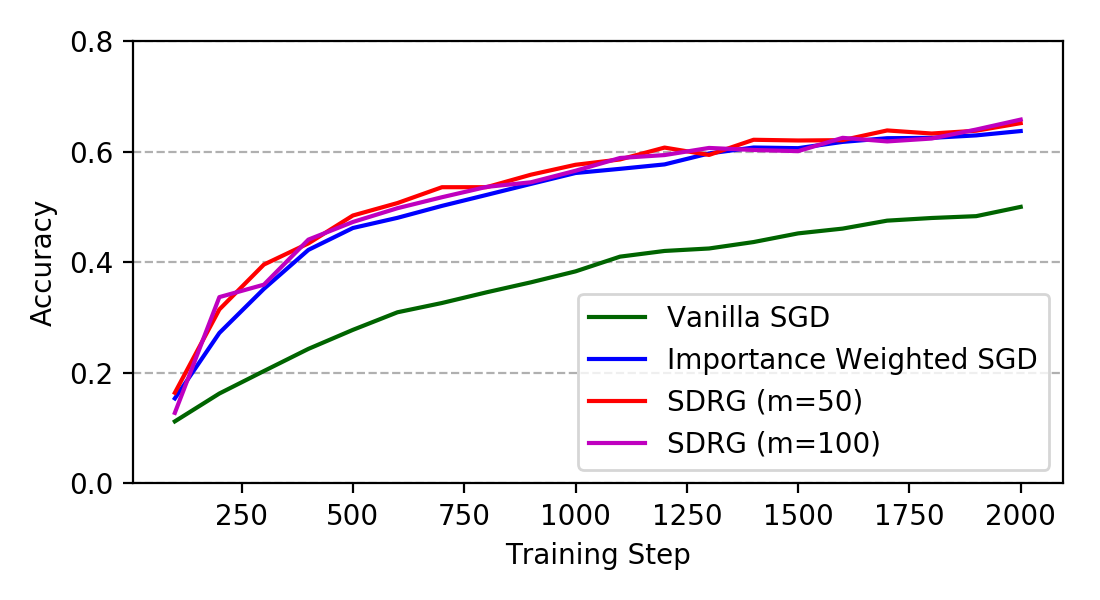}} 
    ~&~
    \subfloat[\label{fig:asdasdf}][Sampling distribution is skewed for the classes in turn (i.e. class 0, 1, 2, ...)]
    {\includegraphics[width=.49\linewidth]{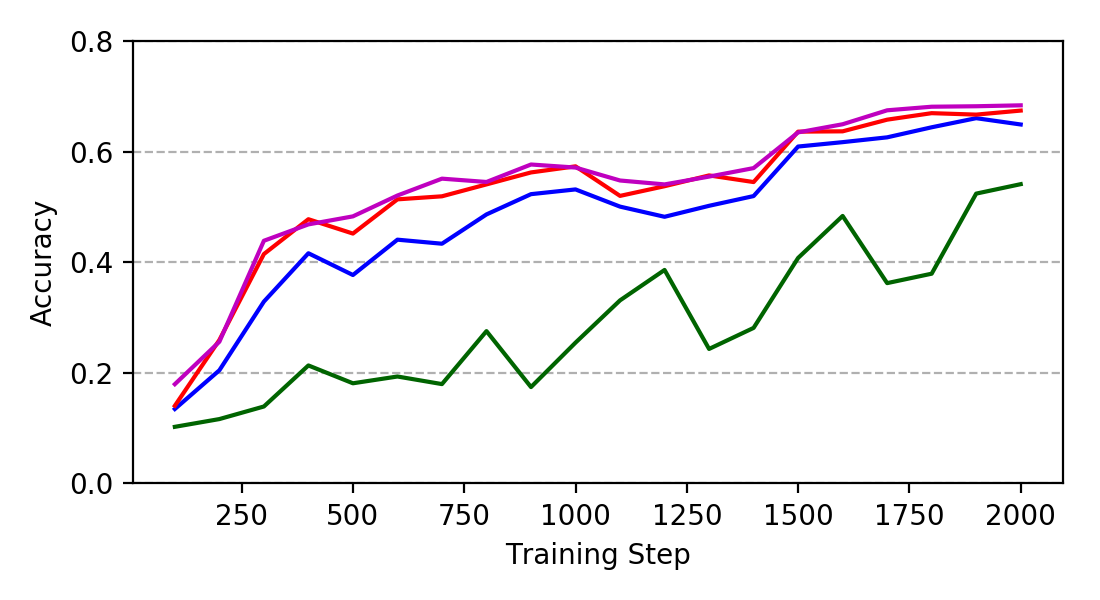}}
  \end{tabular}
  \caption{Accuracy vs training timestep in Fashion-MNIST}
\label{fig:fmnist_results}
\end{figure*}

In both Fig. 1 (a) and (b), SDRG empirically shows a faster convergence rate than importance weighted SGD and vanilla SGD. However, in Fig 1. (b) where the class to be skewed is changed after $m$ iterations, which might be a harder case than (b), the overall convergance rates flucture more than Fig, 1 (a), but SDRG shows more robust convergence than the other algorithms.


\paragraph{Fashion-MNIST}
Fashion-MNIST is an MNIST-like database of clothes. The images are grayscale, with size of 28 x 28 and associated with labels from 10 classes.
We evaluate the performance of SDRG in comparison with importance weighted SGD and vanilla SGD, under experimental settings which are exactly same with the MNIST experiment above.

In Fig 2. we can observe the same tendency in the convergence rate with Fig. 1, where SDRG shows less variant converging pattern for both of cases (a) and (b). 



\section{Conclusion and Future Work}

In this paper, we proposed a SGD algorithm that addresses the caual effects of covariates on the missingness of incomplete data.
Along with the previous studies that extended the use of doubly robust estimators to a variety of machine learning areas 
\cite{ReddiEtal.AAAI15,Kallus.arXiv18,DudikEtal.ICML11,JiangLi.ICML16,ThomasBrunskill.ICML16,FarajtabarEtal.ICML18}, this paper has been the first approach to apply the idea of doubly robust estimator to stochastic optimization.
In SDRG, employing control variates for regression adjustment allowed us to view the proposed method in the framework for variance reduced SGDs that also utilize control variate schemes, except that ours include the additional weight correction term.
For the efficiency in computing and storing control variates, we suggested to choose momentum function as control variates and by doing so, the direct derivation from SDRG to the momentum has been found as a byproduct.
In addition to these notable findings, empirical results have demonstrated that the proposed SDRG shows faster convergence than vanilla SGD and importance weighted SGD.

Future work includes further empirical studies to evaluate the performance of SDRG in various cases of missing data problem: for instance, the setting where multiple covariates are involved in the missingness mechanism. 
Also, theoretical analysis on convergence of SDRG should be provided. 
It would be essential to do deeper investigation on how the property of double robustness affects the convergence of SGD, compared to the existing variance reduced SGDs such as SAGA and SVRG.



\small
\bibliographystyle{abbrv}
\bibliography{paper_khlee}

\end{document}